\documentclass[conference]{IEEEtran}
\usepackage{cite}
\usepackage{amsmath,amssymb,amsfonts,amsthm}
\usepackage{algpseudocode}
\usepackage{algorithmicx}
\usepackage{algorithm} 
\usepackage{graphicx}

\usepackage{tikz}
\usetikzlibrary{automata, positioning, arrows,shapes.multipart,fit}

\usepackage{textcomp}
\usepackage{xcolor}
\newtheorem{theorem}{Theorem}

\newtheorem{definition}{Definition}

\usepackage{ifpdf}
\usepackage[normalem]{ulem}

\newcommand{\blue}[1]{#1}
\newcommand{\norm}[1]{\left\lVert#1\right\rVert}
\newcommand{\rev}[1]{#1}

\begin{document}

\title{On the Differentially Private Nature of Perturbed Gradient Descent}

\author{\IEEEauthorblockN{Thulasi Tholeti}
\IEEEauthorblockA{\textit{Department of Electrical Engineering} \\
\textit{IIT Madras}\\
India \\
ee15d410@ee.iitm.ac.in}
\and
\IEEEauthorblockN{Sheetal Kalyani}
\IEEEauthorblockA{\textit{Department of Electrical Engineering} \\
\textit{IIT Madras}\\
India \\
skalyani@ee.iitm.ac.in}
}

\maketitle

\begin{abstract}
We consider the problem of empirical risk minimization given a database, using the gradient descent algorithm. We note that the function to be optimized may be non-convex, consisting of saddle points which impede the convergence of the algorithm. A perturbed gradient descent algorithm is typically employed to escape these saddle points. We show that this algorithm, that perturbs the gradient, inherently preserves the privacy of the data. We then employ the differential privacy framework to quantify the privacy hence achieved. We also analyze the change in privacy with varying parameters such as problem dimension and the distance between the databases.
\end{abstract}


\section{Introduction}
Given the abundant amount of data openly available about various aspects of an individual, privacy has become one of the major concerns while handling data. Differential privacy is a privacy guarantee on preserving the privacy of an individual when a statistical database is publicly released \cite{dwork2014algorithmic}. When a differentially private mechanism is applied to a pair of databases that differ by a single record, an external agent should ideally not be able to identify the presence or absence of that record. Differential privacy quantifies the extent to which this guarantee is preserved. Differential privacy guarantees are now being provided in various problems such as online learning \cite{jain2012differentially}, Empirical Risk Minimization (ERM) \cite{Kifer2012,zhang2017efficient,Song2013}, boosting \cite{dwork2010boosting}, matrix factorization \cite{berlioz2015applying}, etc. Typically, these mechanisms achieve privacy by adding noise or perturbation at the input, output or an intermediate step in the mechanism \cite{Sarwate2013}. In this work, we quantify the privacy of ERM where the stochastic gradient descent updates are perturbed.

There has been renewed interest in the convergence of first-order iterative optimization methods for non-convex functions. It has been observed in \cite{Jin2017a} that in many non-convex problems such as tensor decomposition, dictionary learning, matrix retrieval, etc., the presence of saddle points impedes the convergence of the stochastic gradient descent (SGD) algorithm. It has also been suggested that local minima can be as good as global minima in high dimensional neural networks but saddle points cause a bottleneck in convergence in \cite{choromanska2015loss}. Therefore, variants of the gradient descent algorithms have been proposed to accelerate the convergence of SGD in the presence of saddles. In \cite{Ge2015}, the authors suggest adding noise from the surface of the unit sphere to the gradient so as to escape saddle points. In \cite{Jin2017a}, noise from a unit sphere is added when the magnitude of the gradient is below a certain threshold. A modified version of \cite{Jin2017a} is also proposed in \cite{daneshmand2018escaping} where gradient descent and SGD are interleaved. In all the above works, we note the following pattern: a sample from an isotropic noise distribution is added to either the gradient or the iterate such that the caused perturbation helps in escaping the saddle point.

As accelerating stochastic gradient descent in the presence of saddle points also involves perturbation, we hypothesize that it should also inherently provide some privacy guarantees. In this work, we quantify the privacy guarantees that are obtained by employing an algorithm that perturbs the gradient with the aim of escaping saddle points. Using the pattern observed in works to escape saddle points, we provide a generic format of the Perturbed gradient descent (PrGD) algorithm. Our major contribution is identifying and quantifying the privacy provided by the PrGD algorithm. We also quantify the privacy obtained by adding noise from a $d$-dimensional ball which, to the best of our knowledge, has not been done before. In the forthcoming section, we discuss some basic definitions from both the optimization and the differential privacy literature. We provide the problem setting and provide a generic algorithm to escape saddle points. We then provide the privacy guarantees of the algorithms and discuss the results.

\section{Definitions and background}
Differential privacy was introduced to formally provide privacy guarantees. \rev{We initially define some terms regarding differential privacy from the seminal work \cite{dwork2014algorithmic}:}
\begin{definition}[Neighbouring databases]
    Two databases $S$ and $S'$ are said to be neighbouring databases if they differ by a single entry. The maximum distance between the databases is denoted by $\Delta x$.
    \begin{equation}
        \Delta x = \max_{x \in S, x' \in S'} \norm{x - x'}^2
    \end{equation}
\end{definition}
\begin{definition}[$(\epsilon,\delta)$-private mechanism]
    A randomized mechanism $\mathcal{M}$ with range $\mathcal{R}$ is said to preserve $(\epsilon,\delta)$ privacy, if for all pairs of neighbouring databases $S$ and $S'$ and for any $\mathcal{A} \subset \mathcal{R}$,
    \begin{equation}
        Pr(\mathcal{M}(S)  \in \mathcal{A}) \leq \exp(\epsilon)  Pr(\mathcal{M}(S')  \in \mathcal{A}) + \delta
    \end{equation}
\end{definition}
Note that when $\epsilon=0$, we get $(0,\delta)$ privacy which is formally defined as follows
\begin{definition}[$(0,\delta)$-private mechanism]
    A randomized mechanism $\mathcal{M}$ with range $\mathcal{R}$ is said to preserve $(0,\delta)$ privacy, if for all pairs of neighbouring databases $S$ and $S'$ and for any $\mathcal{A} \subset \mathcal{R}$,
    \begin{equation}
        |Pr(\mathcal{M}(S)  \in \mathcal{A}) -  Pr(\mathcal{M}(S')  \in \mathcal{A}) | \leq  \delta
    \end{equation}
\end{definition}
The privacy measure, $\delta$, is also known as the total variation distance of the query output for neighbouring databases. Note that for a given $\delta$, $(0,\delta)$ privacy is a stronger guarantee than $(\epsilon,\delta)$ privacy. We also bring to note that a smaller value for \(\delta\) implies greater privacy.

We now define the necessary terms in the optimization framework.
\begin{definition}[Stationary points]
    For a differentiable function $f(·)$, we say that $x$ is a first-order stationary point if $||\nabla f(x)|| = 0$ and a second-order stationary point if $\lambda_{min}(\nabla^2 f(x)) \geq 0$.
\end{definition}
The iterative gradient methods such as gradient descent, stochastic gradient descent, RMSProp, Adam, etc. are first-order methods that guarantee convergence to the first-order stationary point. For a convex objective, convergence to the first-order stationary point guarantees convergence to the minimum. However, for a non-convex function, a first-order stationary point may be a maximum, minimum or a saddle point. It has been observed that the presence of saddle points greatly impede the convergence of gradient descent \cite{choromanska2015loss}.
\begin{definition}[Strict saddle]
    We say $x$ is a saddle point if it is a first-order stationary point, but not a local minimum. Moreover, we say that a saddle point is strict if $\lambda_{min}(\nabla^2 f(x)) < 0$. 
\end{definition}
A strict saddle implies that there is a direction of functional decrease and hence, there is a chance that the perturbation may help the iterate escape the saddle point.

In the forthcoming sections, we discuss the problem setting where we consider a non-convex objective function (this implies that there may be saddle points in play) and a typical perturbed gradient descent algorithm to minimize it. We then quantify the privacy guarantees provided by this algorithm.

\section{Problem setting} 
    We consider the empirical risk minimization problem given a database. We denote the database as a set of points $(x_i,y_i)$ where $x_i \in \mathbb{R}^d$ and $y_i \in \mathbb{R}$ for $i = 1,...N$. We aim to learn the function \(f\) from the given data; let the function \(f\) be parameterized by \(w\) which are tuned using an iterative optimization algorithm. The minimization objective can be written as a function of the inputs $x_i$'s and the parameters $w$. The problem is to minimize the empirical loss $L(w)$ which is given by 
    \begin{equation} \label{eqn:loss}
        L(w) = \dfrac{1}{N} \sum_{i=1}^N \ell(f_w(x_i),y_i)
    \end{equation}
     Here, the loss function $\ell$ is the loss incurred for predicting $f_w(x_i)$ when the given output is $y_i$. 
     We do not assume the convexity of the objective to be minimized. As the objective function $L(w)$ may be non-convex, there arises the problem of saddle points \cite{dauphin2014identifying} which significantly slows down the convergence of the optimizer. We make the following assumptions about the objective:
     \begin{enumerate}
         \item The loss is bounded, is \(\beta\)-smooth and has \(\rho\)-Lipschitz Hessian.
         \item All the saddle points are strict
     \end{enumerate}
     These assumptions are significant in proving the convergence of the perturbed gradient descent algorithm to a local minimum and do not affect the privacy guarantees provided.
    In the process, the privacy of the database also needs to be preserved. Privacy is said to be compromised if an adversary can identify whether an individual entry belongs to the given database based on the output of an algorithm given a pair of neighbouring databases. In our problem, we consider a pair of databases with a maximum distance between their gradients as \(\Delta x\). In the subsequent section, we analyze the privacy guarantees of the iterative optimization algorithm employed to minimize the objective listed in Eq. \ref{eqn:loss}.
    
    \section{Algorithm and guarantees}
    In this section, we list a perturbed version of the stochastic gradient descent algorithm that is employed to minimize a non-convex objective.  Previously, perturbation was added to the gradient to escape saddle points in \cite{Ge2015,Jin2017a,levy2016power}; \blue{the noisy/perturbation added to the gradient is a sample from a unit ball.} Also, from the work done in \cite{Sarwate2013,abadi2016deep,bassily2014private}, we note that privacy can be preserved by adding any noise/perturbation to the gradient values. \blue{Previous works on differential privacy such as \cite{dwork2014algorithmic} deal with the addition of Gaussian or Laplacian noise to the data and computing the resulting privacy. Even in the context of gradient descent, \cite{bassily2014private} deals with the addition of Gaussian noise to the gradient. However, to escape a saddle point, isotropic noise (especially noise sampled from a unit ball) is advocated in \cite{Jin2017a}.} Therefore, in this work, it is shown that both privacy and faster convergence can be achieved in the presence of saddle points when perturbation \blue{sampled from a unit ball} is added to the gradient. A version of the perturbed gradient descent algorithm is presented below:
    \begin{algorithm}[H]   
                \caption{Perturbed gradient descent (PrGD)}    \label{alg:PrGDAlg}
                \begin{algorithmic}[1]
                    \State \textbf{Input}: Initial parameters $w_0$
                    \For { t = 1,2,...T} 
                    \State  Choose a data point $(x_i, y_i)$ uniformly at random from the $N$ available data points
                    \State Sample $n_t$ from the volume of a unit ball 
                    \State $w_{t+1} \leftarrow w_t - \eta(\bigtriangledown \ell(f_w(x_i),y_i) +n_t )$
                    \EndFor
                \end{algorithmic}
            \end{algorithm}
            
    Note that this is the typical format of any perturbed gradient descent algorithm \cite{Ge2015,abadi2016deep} where the distribution of the added noise (usually isotropic) varies. Our major contribution is to show that the PrGD algorithm also inherently provides privacy guarantees. We derive the privacy guarantees provided by Algorithm \ref{alg:PrGDAlg} below.
    \subsection{Privacy guarantees}  
    \blue{Let $\nabla (x_i)_t = \bigtriangledown \ell(f_{w_t}(x_i),y_i)$ denote the gradient for input \(x_i\) at iteration \(t\); for simplicity, we drop the subscripts and denote the gradients of two different inputs \(x,x'\) as \(\nabla x, \nabla x'\) respectively. With a slight overload of notation, we assume that the maximum difference in the gradients $\nabla x$ and $\nabla x'$ is also denoted by $\Delta x$, i.e., $|\nabla x' - \nabla x| \leq \Delta x$.} The farther the points, the easier it is for an adversary to distinguish between them. Therefore, the worst-case analysis is done for $|x' - x| = \Delta x$.
    \begin{theorem}
        \blue{Let \(\Delta x <2\)}. Algorithm \ref{alg:PrGDAlg} provides $(0,\hat{\delta})$- differential privacy guarantees where $$\hat{\delta} = \dfrac{T}{N} \left[ 1 - I_{1-(\Delta x/2)^2} \left( \dfrac{d+1}{2}, \dfrac{1}{2} \right) \right]$$
        where $\Delta x$ denotes the maximum gradient distance between two neighbouring databases \blue{and $I$ is the regularized incomplete beta function defined in  \cite{BetaRegularized}. }
    \end{theorem}
    \begin{proof}
        We approach the quantification of privacy in three steps. We characterize the privacy provided by the addition of a sample from the volume of a unit ball for a single step. Then, we incorporate the fact that a single time step uses only one of the $N$ available data points through random sampling. Finally, the composition theorem is used to characterize the cumulative privacy over the $T$ time steps. 
        \begin{itemize}
            \item Privacy guarantees for addition of noise sampled from the volume of a unit ball: \\
            To establish privacy guarantees, let us initially consider the privacy obtained at an iteration $t$ for a data sample \(x_i\).  As the noise is added from a unit ball, if \(\Delta x\ > 2\), the points can always be distinguished. Therefore, we assume that \(\Delta x\ < 2\). \\
            Let $\nabla \hat{x} = \nabla x +n $ and $\nabla \hat{x}' = \nabla x'+n$ and $|\nabla x' - \nabla x| \leq \Delta x$. Here $n$ is a noise vector sampled from the volume of a unit sphere as demonstrated in Fig. \ref{fig:VolumeUnitBall}. We aim to prove that $|Pr(\nabla \hat{x} \in \Omega) - Pr( \nabla \hat{x}' \in \Omega)| \leq \delta$ for any $\Omega \subseteq \mathbb{R}^d$.
 
The volume of a unit ball in $d$ dimensions is given by the following formula
\begin{equation}
    V_d = \dfrac{\pi^{d/2}}{\Gamma(1+ \frac{d}{2})}.
\end{equation}
\begin{figure}
    \centering
    \begin{tikzpicture}
        \fill[black!40!white, opacity = 0.5] (1.5,0) circle (2cm);
        \fill[black!40!white, opacity = 0.5] (-1,0) circle (2cm);
        \draw[dashed] (-1,3) -- (-1,-3) node[below]{$x'$};
         \draw[dashed](1.5,3) -- (1.5,-3)  node[below]{$x$};
         \draw[dashed] (.25,1.5) -- (.25,-1.5);
         \draw[<->] (-1,2.5) -- (1.5,2.5);
         \node at (0.25,2.5) [above] {$\Delta x$};
    \end{tikzpicture}
    \caption{Distribution of $\hat{x}$ and $\hat{x}'$}
     \label{fig:VolumeUnitBall}
\end{figure}
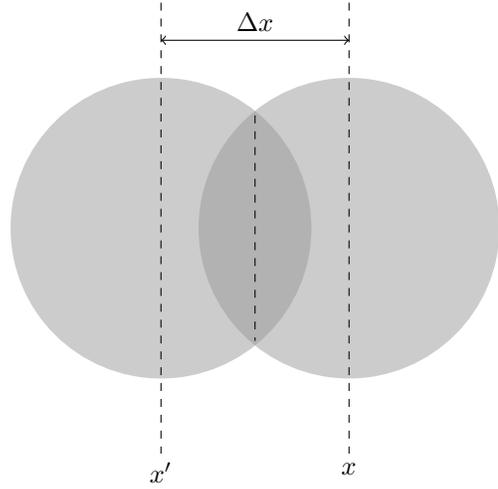

Now, we find the volume of intersection of two hyperspheres with unit radii and centers located at a Euclidean distance of $\Delta x$. This can be obtained as a sum of two hypersphere caps. The volume of a hypersphere cap is derived in \cite{li2011concise} and is given by
\begin{equation*}
    V_d^{cap}(a) = \dfrac{1}{2} \dfrac{\pi^{d/2}}{\Gamma(1+ \frac{d}{2})}  I_{1-a^2} \left( \dfrac{d+1}{2}, \dfrac{1}{2} \right)
\end{equation*}
where \(I\) is the regularized incomplete beta function and $a$ is the difference between the radius of sphere and the height of the cap; here, $a = \frac{\Delta x}{2}$.
\blue{The overlapping volume will be twice of that as the spherical cap as shown in Fig. \ref{fig:VolumeUnitBall}.}
\begin{equation*}
    V_d^{overlap} = \dfrac{\pi^{d/2}}{\Gamma(1+ \frac{d}{2})}  I_{1-a^2} \left( \dfrac{d+1}{2}, \dfrac{1}{2} \right).
\end{equation*}
Consider a set $\Omega \in \mathbb{R}^d$.
\begin{align*}
    P(\nabla \hat{x} \in \Omega) = \dfrac{\text{Volume of set $\Omega$ overlapping the sphere}}{\text{Volume of sphere}}
\end{align*}

\noindent
\rev{The maximum difference between $Pr(\nabla \hat{x} \in \Omega)$ and $Pr(\nabla \hat{x}' \in \Omega)$ is obtained when the set $\Omega$ is the non-overlapping volume of either one hypersphere. The maximum value for \(\delta\) signifies signifies the worst-case privacy.}
\begin{align*}
    \delta &= \dfrac{ \dfrac{\pi^{d/2}}{\Gamma(1+ \frac{d}{2})}  - \text{Overlapping volume}}{ \dfrac{\pi^{d/2}}{\Gamma(1+ \frac{d}{2})} } \\
    &=  \dfrac{ \dfrac{\pi^{d/2}}{\Gamma(1+ \frac{d}{2})}  - \dfrac{\pi^{d/2}}{\Gamma(1+ \frac{d}{2})}  I_{1-a^2} \left( \dfrac{d+1}{2}, \dfrac{1}{2} \right)}{ \dfrac{\pi^{d/2}}{\Gamma(1+ \frac{d}{2})}} \\
    &= 1 - I_{1-a^2} \left( \dfrac{d+1}{2}, \dfrac{1}{2} \right)
\end{align*}
\rev{Using the identity of regularized incomplete beta function \cite{BetaRegularized}, \(I_z(a,b) = 1 - I_{1-z}(b,a)\) we rewrite the above expression as }
\begin{equation} \label{eqn:deltaBall}
    \delta = I_{(\Delta x/2)^2} \left(  \dfrac{1}{2}, \dfrac{d+1}{2} \right).
\end{equation}
\rev{We analyse the effect of variation of the dimension \(d\) and \(\Delta x\) on the privacy \(\delta\) in the next subsection.}
\noindent
\item Effect of random sampling of data points on privacy:\\
As each data point is sampled at random from a set of $N$ data points, we obtain improved privacy guarantees. According to the privacy amplification theorem employed in \cite{abadi2016deep}, the privacy guarantee offered at each step is now $(0,\delta/N)$. 

\item Privacy over $T$ time steps:\\
As the adversary can view multiple input-output pairs over the evolution of the algorithm, there is a compromise on the privacy of the database. This is characterized by the strong composition theorem \cite{kairouz2017composition}. The adaptive composition theorem can be applied when the adversary has information about the databases as well as the mechanism employed by the differentially private agent; in addition, the adversary is allowed to modify its future queries based on the outputs it sees. 
The parameter $w$ for future queries is affected by the output of the differentially private agent. \blue{The adaptive strong composition theorem states that the composition of $K$ mechanisms each providing $(\epsilon_i,\delta_i)$ for $i = 1,...K$ results in a mechanism with privacy \((\sum_{i=1}^K\epsilon_i,\sum_{i=1}^K\delta_i)\).} A direct application of the adaptive composition theorem to our problem results in a $T$-fold composition of equivalent $(0,\delta/N)$ mechanisms. The final guarantee that we get is 
\begin{equation} \label{eqn:overallPrivacy}
    \hat{\delta} = \sum_{i = 1}^T \delta /N = \dfrac{T}{N} \left[ I_{(\Delta x/2)^2} \left(  \dfrac{1}{2}, \dfrac{d+1}{2} \right) \right]
\end{equation}

        \end{itemize}
        \textbf{Note:} When noise sampled from the surface of a unit ball is added to the gradient instead of the volume as done in \cite{Ge2015}, we will not be able to provide privacy guarantees. This is because the output achieved after the addition of noise, i.e., $\hat{x} = x +n $ will be exactly a distance of 1 unit away from $x$. Therefore, the adversary can easily detect which of the two databases contributed to the specific output. Hence, we consider noise sampled from the volume of a unit ball.
    \end{proof}
    
    \subsubsection{Impact of \(d\) and \(\Delta x\) on privacy}
    \rev{The privacy parameter \(\delta\) varies from 0 to 1, where \(0\) corresponds to the case of maximum privacy (when the outputs from neighbouring databases are indistinguishable) and \(1\) corresponds to minimum privacy (where the outputs can be surely distinguished). In this section, we consider the expression for privacy obtained by adding noise sampled from a unit ball as derived in Eq. \ref{eqn:deltaBall} and analyze the effect of the parameters \(\Delta x\) and \(d\) on the privacy metric \(\delta\). Note that as the overall privacy derived in Eq. \ref{eqn:overallPrivacy} is a scaled version of Eq. \ref{eqn:deltaBall}, the same trend applies for the overall privacy as well.}
    
    \rev{From \cite{BetaRegularized}, the expansion of the regularized incomplete beta function is given by }
    \begin{equation}
        I_z(a,n) = z^a \sum_{k=0}^{n-1} \dfrac{(a)_k (1-z)^k}{k!}, \quad n \in \mathbb{N}.
    \end{equation}
    \rev{Applying the above expansion to Eq. \ref{eqn:deltaBall}, under the assumption that \(d\) is odd to ensure that \(n \in \mathbb{N}\), we have }
    \begin{equation} \label{eqn:exp}
        \delta = \dfrac{\Delta x}{2} \sum_{k=0}^{\frac{d-1}{2}} \dfrac{(\frac{1}{2})_k (1-(\frac{\Delta x}{2})^2)^k}{k!}.
    \end{equation}
    \blue{Note that the assumption on \(d\) is only made to study the trend and is not a requirement for Eq. \ref{eqn:overallPrivacy} to hold.} For a fixed \(\Delta x\), let us initially study the impact of the dimension \(d\). Let us consider \(d=1\). This simplifies to the addition of uniform noise from the interval \([-1,1]\). On substituting \(d=1\) in Eq. \ref{eqn:exp}, we obtain \(\delta = \Delta x/2\) which corroborates with the result obtained for \((0,\delta)\)-privacy in case of a uniform distribution in \cite{He2017}. To analyse the variation of the quantity \(\delta\) with \(d\), let us consider \(d = 1,3,5,\cdots\). For \(d=3\), we obtain terms corresponding to \(k=0,1\) in the summation. Note that all the terms in the summation are positive and hence, as \(d\) increases, more terms get added to the summation, the value of \(\delta\) increases for the same value of \(\Delta x\). As \(\Delta x/2 < 1\), we observe that \( (1-(\frac{\Delta x}{2})^2)^k\) decreases with increasing \(k\). This results in an increase in \(\delta\) thereby resulting in decrease of privacy. 
    
    \rev{We then analyse how \(\delta\) varies with the quantity \(\Delta x\). Note that \(\dfrac{\Delta x}{2}\) is an increasing quantity in \(\Delta x\) whereas \((1-(\frac{\Delta x}{2})^2)\) is decreasing in \(\Delta x\). Therefore, we rely on the sign of the gradient to characterize if \(\delta\) is an increasing or a decreasing function in \(\Delta x\). Using the differentiation formulas for Beta regularized functions in \cite{BetaRegularized}, we have}
    \begin{equation} \label{eqn:betaincDiff}
        \dfrac{d I_{(\frac{\Delta x}{2})^2} (\frac{1}{2},\frac{d+1}{2})}{d(\frac{\Delta x}{2})^2} = \dfrac{(1-(\frac{\Delta x}{2})^2)^{\frac{d-1}{2}} (\frac{\Delta x}{2})^{-1}}{B(\frac{1}{2},\frac{d+1}{2})},
    \end{equation}
    \rev{where \(B(a,b)\) denotes the beta function as defined in \cite{BetaFnc}.
    We note that all the terms in Eq. \ref{eqn:betaincDiff} are positive for all values of \(d\). As the gradient is positive, we can conclude that \(\delta\) is an increasing function of \(\Delta x\) for any given \(d\). These observations are further confirmed by plotting the variation of \(\delta\) with \(\Delta x\) for different dimensions in Fig \ref{fig:deltaFig}.}
    
      \begin{figure}[ht]
        \centering
        \includegraphics[scale = 0.5]{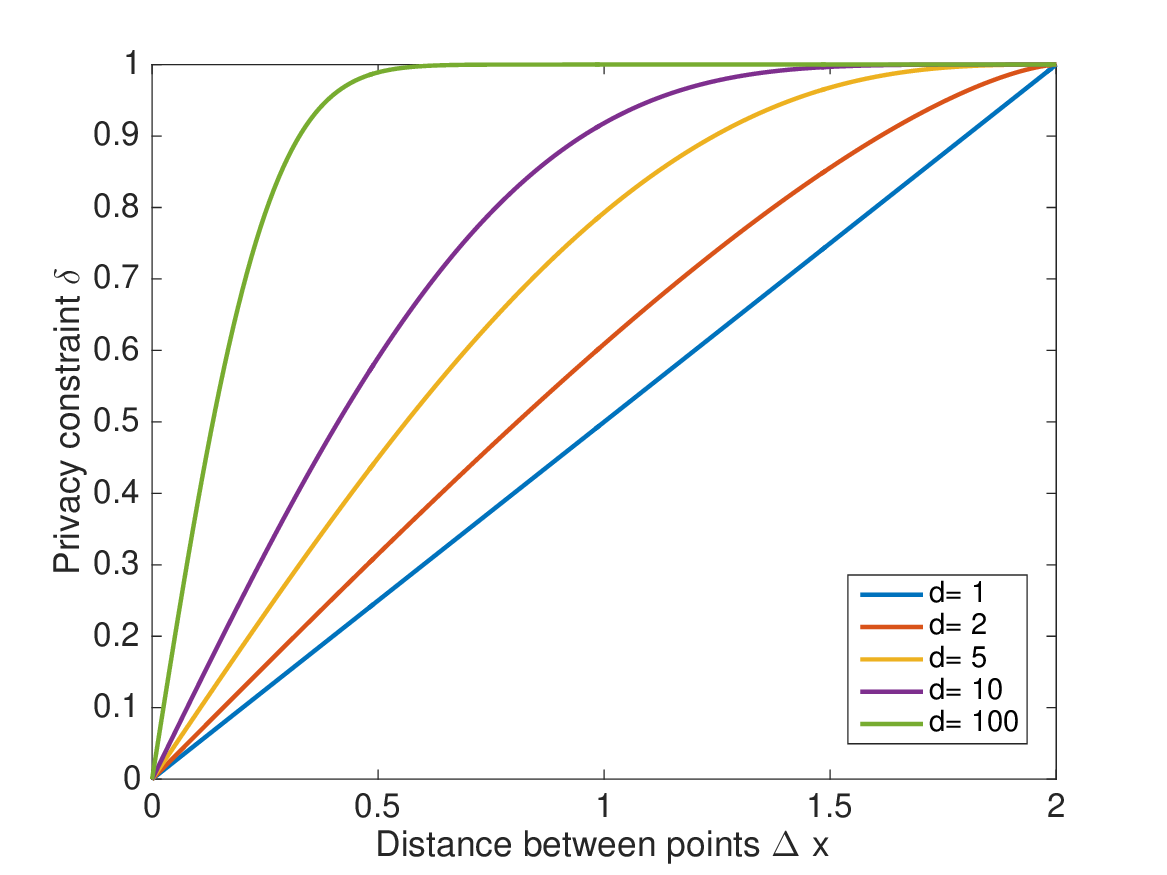}
        \caption{Variation of privacy with \(\Delta x\) for different \(d\)}
        \label{fig:deltaFig}
    \end{figure}
    
    \subsection{Convergence guarantees}
    We can guarantee the convergence of the SGD algorithm arbitrarily close to a local minimum by following the analysis in \cite{Ge2015}.
    The result in \cite{Ge2015} is reproduced here for convenience:
    \begin{theorem}
    Suppose a function $f(w) : \mathbb{R}^d \rightarrow \mathbb{R}$ that is strict saddle, and has a stochastic gradient oracle where the noise satisfy $\mathbb{E}[nn^T] = \sigma^2 I$ for some $\sigma$. Further, suppose the function is bounded by $|f(w)| \leq B$, is $\beta$-smooth and has $\rho$-Lipschitz Hessian. Then there exists a threshold $\eta_{max}$, so that for any $\zeta >0$, and for any $\eta \leq \eta_{max}/\max\{1, log(1/\zeta)\}$, with probability at least $1 - \zeta$ in $t = O(\eta^{-2} log(1/\zeta))$ iterations, SGD outputs a point $w_t$ that is $O(\sqrt{ \eta log(1/\eta \zeta)})$-close to some local minimum $w^*$.
\end{theorem}

    The property of strict saddle and equivalence of local and global minima is commonly encountered, especially when we deal with an over-parameterized shallow neural network \cite{du2018power}. Also, as the noise is sampled uniformly at random from the volume of a unit ball, it is isometric. That automatically satisfies the requirement $\mathbb{E}[nn^T] = \sigma^2 I$. Therefore, Algorithm \ref{alg:PrGDAlg} is guaranteed to output a point that is close to a local minimum.
    
    \subsection{Discussion}
     The privacy guarantee provided is $(0,\delta)$ and not $(\epsilon,\delta)$, as the perturbation added was from a bounded distribution, the unit ball. The unit ball for a single dimension corresponds to the uniform distribution. \rev{Here, we quantified the privacy achieved by adding noise from a unit ball. However, if we wish to address the alternate problem, i.e., if a certain privacy guarantee is desired, the radius of the $d$-dimensional ball can be appropriately scaled to achieve it.} The convergence guarantees of the algorithm still holds as long as the noise added is isotropic.
    
    \section{Conclusion}
    \rev{This work aims to bring out the inherent privacy provided by the algorithm which perturbs for saddle point escape in a non-convex setting and hence the factor $\delta$ for a typical saddle point escape algorithm is derived. Our major contribution lies in quantifying the \((0,\delta)\) privacy achieved by adding noise randomly sampled from a \(d\)-dimensional ball which has not been attempted before. We also analyze the change in privacy with varying dimensions. We also quantify the overall privacy obtained when the PrGD algorithm is applied to a database over \(T\) time steps while providing convergence guarantees.}
\bibliography{library.bib}
\bibliographystyle{IEEEtran}

\end{document}